\documentclass[conference]{IEEEtran}
\IEEEoverridecommandlockouts
\usepackage{cite}
\usepackage{amsmath,amssymb,amsfonts,amsthm}
\usepackage{algorithmic}
\usepackage{graphicx}
\usepackage{textcomp}
\usepackage{url}                       
\usepackage[table, svgnames]{xcolor}   
\usepackage{xfrac}                     
\usepackage{booktabs}
\usepackage{multicol}

\def\BibTeX{{\rm B\kern-.05em{\sc i\kern-.025em b}\kern-.08em
    T\kern-.1667em\lower.7ex\hbox{E}\kern-.125emX}}

\newcommand{\inv}{^{\raisebox{.4ex}{$\scriptscriptstyle-1$}}}
\newcommand{\invsqrt}{^{\raisebox{.4ex}{$\scriptscriptstyle\hspace{-1pt}-\hspace{-1pt}\frac{1}{2}$}}}
\newcommand{\nicesqrt}{^{\raisebox{.4ex}{$\scriptscriptstyle\frac{1}{2}$}}}

\newcommand{\github}[0]{\url{https://github.com/KrishnaswamyLab/hypergraph-wavelets}}

\makeatletter
\newcommand{\linebreakand}{%
  \end{@IEEEauthorhalign}
  \hfill\mbox{}\par
  \mbox{}\hfill\begin{@IEEEauthorhalign}
}
\makeatother

\begin{document}

\title{Hyperedge Representations with Hypergraph Wavelets: Applications to Spatial Transcriptomics
\thanks{
\hspace{-12pt} 
$^\dag$Denotes equal contribution. \\
\textbf{Correspondence: }\texttt{smita.krishnaswamy@yale.edu}.\\
\textbf{Funding:} NSF Career Grant~2047856, NIH~1R01GM130847-01A1, NIH~1R01GM135929-01, and NSF~OIA~2242769.
}
}

\author{
\IEEEauthorblockN{Xingzhi Sun$^\dag$}
\IEEEauthorblockA{\textit{Yale University}}
\and
\IEEEauthorblockN{Charles Xu$^\dag$}
\IEEEauthorblockA{\textit{Yale University}}
\and
\IEEEauthorblockN{Jo\~ao F. Rocha}
\IEEEauthorblockA{\textit{Yale University}}
\and
\IEEEauthorblockN{Chen Liu}
\IEEEauthorblockA{\textit{Yale University}}
\linebreakand
\IEEEauthorblockN{Benjamin Hollander-Bodie}
\IEEEauthorblockA{\textit{Yale University}}
\and
\IEEEauthorblockN{Laney Goldman}
\IEEEauthorblockA{\textit{Harvey Mudd College}}
\and
\IEEEauthorblockN{Marcello DiStasio}
\IEEEauthorblockA{\textit{Yale University}}
\and
\IEEEauthorblockN{Michael Perlmutter}
\IEEEauthorblockA{\textit{Boise State University}}
\and
\IEEEauthorblockN{Smita Krishnaswamy}
\IEEEauthorblockA{\textit{Yale University}}
}

\newtheorem{theorem}{Theorem}[section]
\newtheorem{lemma}[theorem]{Lemma}
\newtheorem{prop}[theorem]{Proposition}
\newtheorem{cor}{Corollary}
\newtheorem{definition}{Definition}[section]
\newtheorem{conj}{Conjecture}[section]
\newtheorem{rem}{Remark}
\newtheorem{proposition}{Proposition}
\maketitle

\begin{abstract}



In many data-driven applications, higher-order relationships among multiple objects are essential in capturing complex interactions. Hypergraphs, which generalize graphs by allowing edges to connect any number of nodes, provide a flexible and powerful framework for modeling such higher-order relationships. In this work, we introduce hypergraph diffusion wavelets and describe their favorable spectral and spatial properties. We demonstrate their utility for biomedical discovery in spatially resolved transcriptomics by applying the method to represent disease-relevant cellular niches for Alzheimer's disease.
\end{abstract}


\section{Introduction}
Graph representation learning methods such as graph embedding~\cite{node2vec}, graph kernel methods~\cite{Weisfeiler_Lehman}, and graph neural networks~\cite{GCN, GAT, GraphSAGE} have significantly improved our ability to model graph-structured data. However, traditional graphs are limited in that they only capture pairwise interactions 
which often fail to model the higher-order relationships inherent in real-world data. To address this issue, we consider hypergraphs~\cite{HypergraphTheory} that generalize graphs by allowing hyperedges to connect multiple nodes thus allowing a more flexible representation of complex data. Despite growing interest in hypergraphs, research efforts have mainly focused on node-level or graph-level tasks, while hyperedge representation learning remains comparatively unexplored~\cite{HGNN, HGCN_HGAT}.

Hyperedge representations are critical in scenarios where interactions involve more than two nodes. In social networks, hyperedges can model group interactions, such as multi-user collaborations or events. In recommendation systems, hyperedges can represent group preferences, reflecting collective influences that go beyond individual choices. Similarly, in biological networks, they capture complex interactions among multiple proteins or molecules within pathways. These examples highlight the importance of hyperedge representations for modeling the nature of interactions in graph-structured data.

The relevance of hyperedges is particularly pronounced in spatial transcriptomics~\cite{ST_1, ST_2, ST_3}, which maps cellular gene expression within the spatial context of tissues. Traditional graph representations often treat cells as nodes and their interactions as edges, overlooking the broader concept of the cellular niche, i.e., the neighborhoods in which cells exist and their interactions with one another. Hyperedges provide a natural way to model these cellular niches and enhance understanding of cellular behaviors such as cell growth and disease progression.

In this work, we introduce hypergraph diffusion wavelets as a framework for hyperedge representation learning. We will present the framework and describe its favorable spectral, spatial, and computational properties. Finally, we apply our approach to spatial transcriptomics data and demonstrate its potential to capture complex cellular niches and advance biomedical discovery.

\section{Hypergraph Preliminaries}


\subsection{Hypergraph representation}
\label{sec: hyperpreliminaries}

A hypergraph $\mathcal{G}=(\mathcal{V},\mathcal{E})$ is a generalized graph in which generalized edges, referred to as hyperedges, can contain more than two nodes. That is, $\mathcal{V}=\{ v_1,\ldots,v_n \}$ is a collection of nodes (also called vertices) and $\mathcal{E} = \{ e_1,\ldots,e_m \}$ is a collection of subsets $\mathcal{V}$ referred to as hyperedges. (In the case where each hyperedge $e_j$ contains exactly two vertices, then $\mathcal{G}$ is equivalent to an ordinary graph.)


Each hypergraph $\mathcal{G}$ with $n$ vertices and $m$ edges can be represented by an $n\times m$ incidence matrix $\mathbf{H}$ defined by 
\begin{equation}
\mathbf{H}[i,j] = 
\begin{cases}
    1 &\text{ if } v_i \in e_j \\
    0 &\text{ if } v_i \notin e_j
\end{cases}.
\label{eqn:incidence_matrix}
\end{equation}
The degree $\text{deg}(v_i)$ of a vertex $v_i$ is defined to be the number of $e_j$ such that $v_i\in e_j$, and the degree $\text{deg}(e_j)$  of a hyperedge $e_j$ is the number of $v_i\in e_j$ (i.e., the cardinality of $e_j$). We let $\mathbf{D}_{\mathcal{V}}$ be the $n\times n$ diagonal vertex-degree matrix with $\mathbf{D}_{\mathcal{V}}[i,i]=\text{deg}(v_i)$ (and $D_{\mathcal{V}}[i,j]=0$ if $i\neq j$) and let $\mathbf{D}_{\mathcal{E}}$ denote the $m\times m$ diagonal edge-degree matrix defined similarly.
In this work, we consider unweighted hypergraphs but note that our methods can be readily extended to the weighted case. 


\subsection{Bipartite expansion}

The bipartite expansion of a hypergraph $\mathcal{G} = (\mathcal{V}, \mathcal{E})$, with $|\mathcal{V}|=n$, $|\mathcal{E}|=m$, is a bipartite graph $\widetilde{\mathcal{G}} = (\widetilde{\mathcal{V}}, \widetilde{\mathcal{E}})$, where the vertex set $\widetilde{\mathcal{V}}=\mathcal{V} \cup \mathcal{E}$ consists of both the vertices and the hyperedges of the original hypergraph $\mathcal{G}$.
The edge set $\widetilde{\mathcal{E}}$ of the expanded graph connects the nodes to their corresponding hyperedges in the original hypergraph. That is each edge $\widetilde{e} \in \widetilde{\mathcal{E}}$ has the form $\widetilde{e}=(v,e)$ where $v\in \mathcal{V}$ was a vertex and $e\in\mathcal{E}$ was a hyperedge in the original hypergraph $\mathcal{G}$. 

\begin{figure}[!htb]
    \centering
\includegraphics[width=0.35\textwidth]{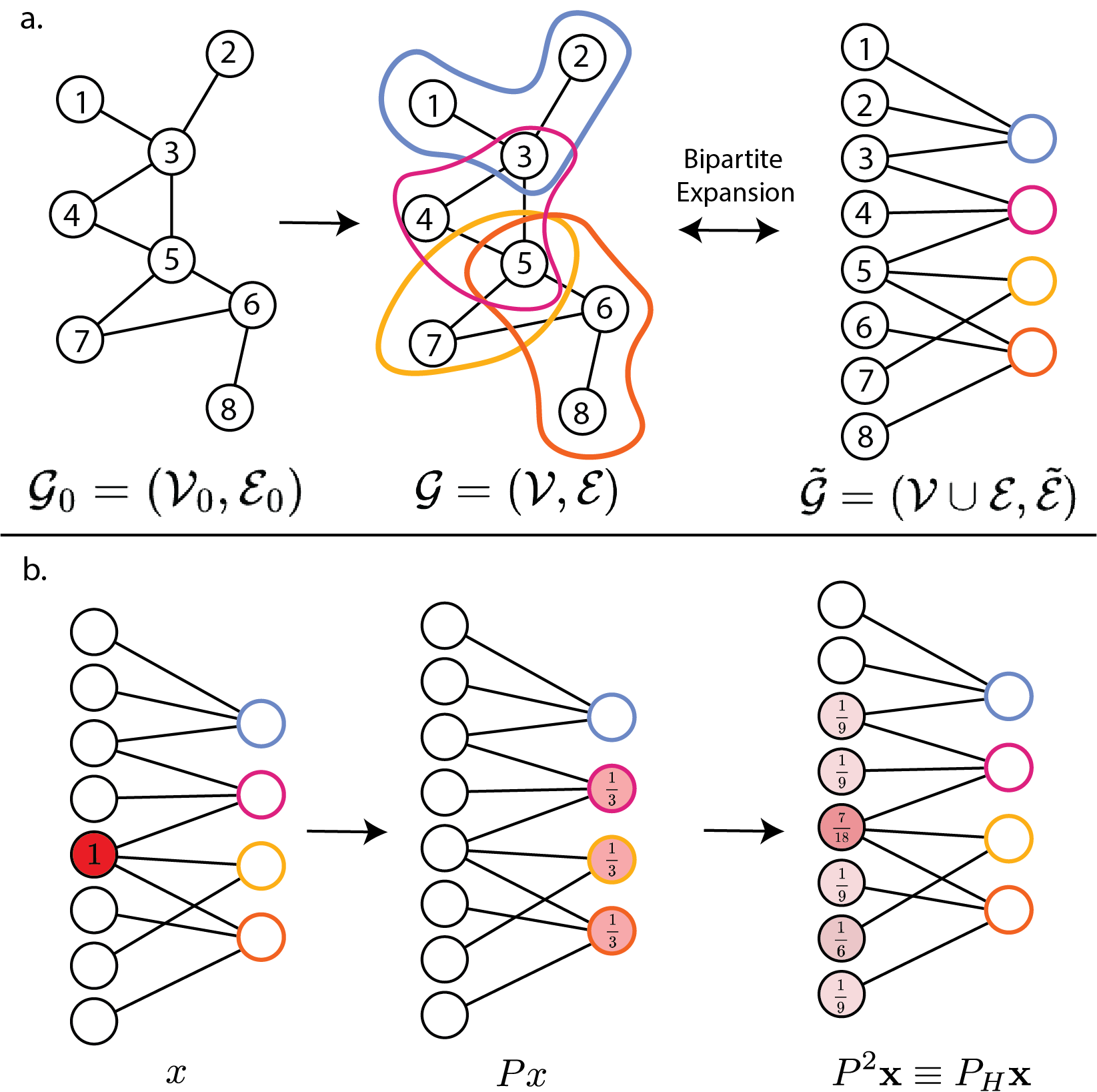}
    \caption{Conceptual illustrations. \textbf{a.} Creation of a hypergraph and the bipartite expansion. \textbf{b.} Diffusion (Lazy Random Walks) on the expanded graph.}
    \label{fig:framework}
\end{figure}


We will let $\widetilde{\mathbf{A}}$ and $\widetilde{\mathbf{D}}$ denote the $(n+m)\times(n+m)$ adjacency and degree matrices of $\widetilde{\mathcal{G}}$ and observe that by construction
\begin{align*}
\widetilde{\mathbf{A}} = 
\begin{bmatrix} 
 0 & \mathbf{H} \\
 \mathbf{H}^\top & 0
\end{bmatrix},\quad \widetilde{\mathbf{D}} = 
\begin{bmatrix} 
 \mathbf{D}_{\mathcal{V}} & 0 \\
 0 & \mathbf{D}_\mathcal{E} \\
\end{bmatrix},
\end{align*}
where $\mathbf{H}$, $\mathbf{D}_\mathcal{V}$, and $\mathbf{D}_\mathcal{E}$ are as in Section \ref{sec: hyperpreliminaries}.

The corresponding symmetrically normalized graph Laplacian is then given by
$\widetilde{\mathbf{L}}_{N} = \mathbf{I} - \widetilde{\mathbf{D}}\invsqrt \widetilde{\mathbf{A}} \widetilde{\mathbf{D}}\invsqrt.$
It is known~\cite{chung1997spectral} that $\widetilde{\mathbf{L}}_N$ is positive semi-definite and can be factorized as $\widetilde{\mathbf{L}}_N = \mathbf{V} \mathbf{\Omega} \mathbf{V}^\top$, where $\mathbf{V}$ is the unitary matrix whose columns are the eigenbasis composed of $\widetilde{\mathbf{L}}_{N}$'s eigenvectors and $\mathbf{\Omega}$ is the diagonal matrix whose entries are $\widetilde{\mathbf{L}}_{N}$'s eigenvalues.


\section{Methods}

In this section, we describe an established \cite{zhou2006learning} hypergraph diffusion operator using the bipartite hypergraph expansion. We then provide some additional characterization of the spectral properties of this diffusion operator and implementation details that make it feasible to process large hypergraphs. 

\subsection{Hypergraph Diffusion}

The graph random walk diffusion operator is defined as:
\begin{align}
\widetilde{\mathbf{P}} &= \widetilde{\mathbf{A}}\widetilde{\mathbf{D}}\inv = \mathbf{I} - \widetilde{\mathbf{D}}\nicesqrt \widetilde{\mathbf{L}}_N \widetilde{\mathbf{D}}\invsqrt 
= \widetilde{\mathbf{D}}\nicesqrt \left( \mathbf{I} - \widetilde{\mathbf{L}}_N \right) \widetilde{\mathbf{D}}\invsqrt\nonumber\\
&= \widetilde{\mathbf{D}}\nicesqrt \mathbf{V} \left( \mathbf{I} - \mathbf{\Omega} \right) \mathbf{V}^\top \widetilde{\mathbf{D}}\invsqrt.\label{eqn: I minus L}
\end{align}
We may also write the diffusion operator:
\begin{align*}
\widetilde{\mathbf{P}} =
\widetilde{\mathbf{A}}\widetilde{\mathbf{D}}\inv \hspace{-.05in}&=
\begin{bmatrix} 
 0 & \mathbf{H} \\
 \mathbf{H}^\top & 0 \\
\end{bmatrix}\hspace{-.08in}
\begin{bmatrix} 
 \mathbf{D}_\mathcal{V}\inv & 0 \\
 0 & \mathbf{D}_\mathcal{E}\inv \\
\end{bmatrix}
\hspace{-.05in}=\hspace{-.05in}
\begin{bmatrix} 
 0 & \mathbf{H}\mathbf{D}_\mathcal{E}\inv \\
 \mathbf{H}^\top \mathbf{D}_\mathcal{V}\inv & 0 \\
\end{bmatrix}
\end{align*}

which implies
\begin{align}\label{eqn: P2}
\widetilde{\mathbf{P}}^2
&=
\begin{bmatrix} 
 \mathbf{H}\mathbf{D}_\mathcal{E}\inv \mathbf{H}^\top \mathbf{D}_\mathcal{V}\inv & 0 \\
 0 & \mathbf{H}^\top \mathbf{D}_\mathcal{V}\inv \mathbf{H}\mathbf{D}_\mathcal{E}\inv \\
\end{bmatrix}.
\end{align}
We observe that $\widetilde{\mathbf{P}}$ is a block diagonal matrix, which is a consequence of $\widetilde{\mathcal{G}}$ being bipartite, and let 
\begin{equation}
\mathbf{P}_H := \mathbf{H}\mathbf{D}_\mathcal{E}\inv \mathbf{H}^\top \mathbf{D}_\mathcal{V}\inv
\label{eqn: PH}
\end{equation}
denote the top left block. 

To better understand $\mathbf{P}_H$, we note that in the case where the $\mathcal{G}$ is a graph (i.e., each hyperedge contains exactly two vertices), $\mathbf{P}_H$ coincides with the transition matrix of a lazy random walk on the vertices of $\mathcal{G}$. That is,
\begin{equation}
\mathbf{P}_H = \frac{1}{2}(\mathbf{I} + \mathbf{A}_\mathcal{V}\mathbf{D}_\mathcal{V}\inv),
\end{equation}
where $\mathbf{D}_\mathcal{V}$ and $\mathbf{A}_\mathcal{V}$ are the vertex-degree and vertex-adjacency matrices of $\mathcal{G}$.
To verify this, we note that each edge $e_j\in E$, has degree two and thus $\mathbf{H}\mathbf{H}^\top = \mathbf{D}_\mathcal{V} + \mathbf{A}_\mathcal{V}$. Therefore, 
\begin{equation*}
\mathbf{P}_H
= \mathbf{H} (\frac{1}{2} \mathbf{I}) \mathbf{H}^\top \mathbf{D}_\mathcal{V}\inv
= \frac{1}{2} \mathbf{H}\mathbf{H}^\top \mathbf{D}_\mathcal{V}\inv 
= \frac{1}{2} \left(\mathbf{I} + \mathbf{A}_\mathcal{V} \mathbf{D}_\mathcal{V}\inv \right).
\end{equation*}

Thus, as observed in \cite{zhou2006learning}, we may interpret $\mathbf{P}_H$ as a hypergraph random walk matrix. In particular, $\mathbf{P}_H$ describes the transition probabilities of a walker on the vertices $\mathcal{V}$ who at each step first chooses a random edge $e$ which is incident to the current vertex $v$ and then chooses a random vertex $v'\in e$. 

To further illustrate the connection between $\mathbf{P}_H$ and random walk matrices, we establish the following result. \begin{proposition}\label{lem: eigenvalues}
The eigenvalues of $\mathbf{P}_H$ are contained in the interval $[0,1]$. \end{proposition}
\begin{proof}
It is known (see e.g., \cite{chung1997spectral}) that the eigenvalues of $\mathbf{\widetilde{L}}_N$ lie in $[0,2],$ which implies that $\mathbf{I}-\mathbf{\widetilde{L}}_N$ has eigenvalues in $[-1,1]$. Since similar matrices have the same eigenvalues,  \eqref{eqn: I minus L} implies that $\mathbf{\widetilde{P}}$ has eigenvalues in $[-1,1]$ and thus $\mathbf{\widetilde{P}}^2$ has eigenvalues in $[0,1]$. The result now follows from \eqref{eqn: P2} since $\mathbf{P}_H$ is the upper left block of $\mathbf{\widetilde{P}}^2$. 
\end{proof}


We also note that similar to traditionally random walks, hypergraph diffusion is spatially localized. A random walker can travel at most distance $d$ in $d$ steps of diffusion, where the hypergraph distance between nodes $v_i, v_j$ is given by the minimal number of hyperedges in a path $p$ of the form $p=(v_{p_1}, e_{p_1}, v_{p_2}, e_{p_2}, \ldots v_{p_{d+1}})$ with $v_{p_1} = v_i$ and $v_{p_2} = v_j$ and we assume $v_i,v_{i+1}\in e_i$. However, if the hypergraph $\tilde{\mathcal{G}}$ is constructed from a $k$-hop lifting map (see section \ref{sub:hypergraph_rep}) applied to a non-hyper graph $\mathcal{G}_0$, the receptive field of $d$ steps of hyper-graph diffusion has length $kd$ on the non-hyper graph $\mathcal{G}_0$. Thus, lifting a graph $\mathcal{G}_0$ to a hypergraph $\mathcal{G}$ allows for information to spread more rapidly in space.  

\subsection{Hypergraph Diffusion Wavelets}
\label{subsec:hypergraph_diffusion_wave}


We construct a hypergraph diffusion wavelets bank using the hypergraph diffusion operator $\mathbf{P}_H$ defined as in \eqref{eqn: PH}. In particular, these wavelets will be constructed using the differences of various powers of $\mathbf{P}_H$. Given a sequence of increasing integers $s_0 \leq s_1\leq\ldots\leq s_J$, with $s_0=0$ and $s_1 =1$, which are interpreted as diffusion scales, we define a hypergraph wavelet filter bank $\mathcal W_J = \{ \Psi_i \}_{i = 0}^{J-1}\cup\{\Phi_J\}$:
\begin{align*}
    \Psi_i = \mathbf{P}_H^{s_i} - \mathbf{P}_H^{s_{i+1}},\quad\Phi_J = \mathbf{P}_H^{s_J},
\end{align*}
and for a signal (function) $\mathbf{x}:V\rightarrow \mathbb{R}$, identified with the vector $\mathbf{x}$ with entries $x_i=\mathbf{x}(v_i)$, we will let $\mathcal{W}_J\mathbf{x}=\{ \Psi_i\mathbf{x} \}_{i = 0}^{J-1}\cup\{\Phi_J\mathbf{x}\}$.
These wavelets hence can capture local and global information. The multi-scale nature of these wavelets can be tuned by choice of the scale sequence $s_0, s_1, \ldots, s_J$. 

\subsection{Complexity}

Compared to other hypergraph signal processing methods, our approach is fast and memory efficient, therefore making it amenable to large hypergraphs and datasets. Notably, we compute the vectors $\Psi_i\mathbf{x}=\mathbf{P}_H^{s_i}\mathbf{x}-\mathbf{P}_H^{s_{i+1}}\mathbf{x}$ through sparse matrix-vector multiplications and vector-vector subtractions (see \cite{tong2024learnable}, Section V). Therefore, per signal, the wavelet transform has complexity $O((\bar{d_v} n + \bar{d_e} m)s_J)$ where $\bar{d_v}$ and $\bar{d_e}$ are the average vertex and hyperedge degree respectively. Thus there are significant computational savings when compared with hypergraph signal processing methods which rely on eigendecomposition or tensor decomposition, \textcolor{black}{and can have complexity $O(n^2 + n\log(d_v^\star-1))$ where $d_e^\star$ denotes the maximum edge degree }\cite{zhang2019introducing}. 

\section{Modeling Spatial Transcriptomics Data}
Spatial transcriptomics has emerged as a powerful tool for biomedical discovery, enabling the measurement of spatially-resolved RNA transcript abundances within tissues~\cite{ST_1, ST_2, ST_3}. Given the critical role of spatial organization in biological processes, mapping molecular pathways within their native spatial contexts has driven significant research, as well as diagnostic and therapeutic advances~\cite{ST_tissueArch, ST_tumor_morphology, ST_disease_traj}.  Although techniques for analyzing spatial transcriptomic data have expanded in recent years, existing methods still lack the capability to generate unsupervised multiscale representations of tissue contexts~\cite{ST_challenges, GraphST}. Such representations are essential for making meaningful inferences about the biological implications of coordinated cellular functions in different samples or conditions~\cite{chang2022define, zahedi2024deep, gao2024review}.

\subsection{Hypergraph Modeling of Spatial Transcriptomic}
\label{sub:hypergraph_rep}
Existing graph-based methods for analyzing spatial transcriptomics data utilize cell-cell graphs~\cite{wu2022space}. In tissues, however, cells frequently interact with a larger functional unit that contains more than their adjacent cells~\cite{de2013functional}. We call this unit the \textit{cellular niche}.

\begin{figure}[!htb]
    \centering
    \includegraphics[width=0.5\textwidth]{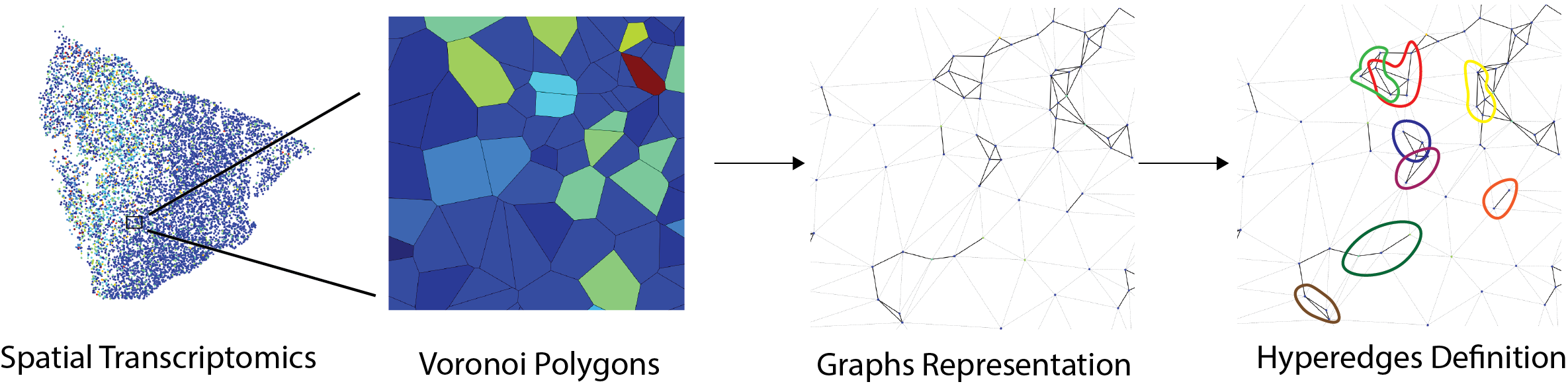}
    \caption{We represent spatial transcriptomics data as a hypergraph. Physically proximate collections of cells play functional roles, motivating the use of hyperedges to represent cellular niches.}
    \label{fig:hypergraph_construction}
\end{figure}



The spatial transcriptomics data $(\mathbf{X}, \mathbf{M}, y)$ consists of the spatial location matrix $\mathbf{X} = \left(x_i^{(j)} \right)_{n\times 2}$,  gene expression matrix $\mathbf{M} = \left(c_{i}^{(k)} \right)_{n\times q}$, where $c_i^{(k)}$ denotes the expression level of the $k$-th gene in the $i$-th cell, $n$ denotes the number of cells, and $q$ the number of genes. Additionally, cell type labels $y=(y_1,\dots,y_n)^\top$ are assumed to be given for each cell. These may be computed from transcriptional data using a variety of approaches~\cite{Seurat, SingleR, scBERT, wang2022benchmarking}.


In a process illustrated in Figure \ref{fig:hypergraph_construction}, we begin by constructing a cell-cell graph $\mathcal{G}_0=(\mathcal{V}_0, \mathcal{E}_0)$, where the vertices $\mathcal{V}_0$ are the cells. Similar to \cite{wu2022space}, the edge set $\mathcal{E}_0$ is constructed by first considering the Voronoi decomposition about each of the cells and then placing edges between pairs of cells whose regions are physically adjacent. (Alternatively, the cell-cell graph could be constructed as a $k$-NN graph or via Delaunay triangulation \cite{palla2022squidpy}.) $\mathcal{G}_0$ is then lifted into a hypergraph $\mathcal{G} = (\mathcal{V}, \mathcal{E})$, where the vertex set $\mathcal{V} = \mathcal{V}_0$ is still the cells and the edges are defined via $k$-hop neighborhoods on $\mathcal{G}_0$. That is, the edges $e\in \mathcal{E}$ are all the sets of the form $e=\{w: d(v,w)\leq k$\} for some $v\in \mathcal{V}$ (where $d(v,w)$ denotes the length of the shortest path between $v$ and $w$, with $d(v,v)=0$). In practice, we find that $k=3$ provides good representations. 



\subsection{Hyperedge Features}
Features for the hyperedges are created to incorporate the available transcriptional, spatial, and cell type information. As is customary for single-cell analysis, transcriptional information in the form of normalized, log-transformed gene expression counts are utilized as a starting point. Within each hyperedge, we compute the following features:
\begin{enumerate}
    \item The average of each transformed gene count in the neighborhood. \label{hyp_feature: average}
    \item The correlation between pairs of genes in the hyperedge.\label{hyp_feature: correlation}
    \item The correlation between transformed gene counts and one step of diffusion applied to transformed counts. \label{hyp_feature: diff_correlation} 
    \item Cell type count. Here, three granularities of the cell label are used: cell type, subclass, and supertype. At each level of granularity, the counts for each type are summed. \label{hyp_feature: cell_type}
\end{enumerate}

These features allow for the natural incorporation of niche information. Feature \ref{hyp_feature: average} directly captures transcriptional information, while Features \ref{hyp_feature: correlation} and \ref{hyp_feature: diff_correlation} provide a sense of coexpression and spatial localization of genes. Feature \ref{hyp_feature: cell_type} makes use of hierarchical cell type knowledge and provides a high-level view of neighborhood composition. These features are stored in a matrix denoted $\mathbf{z}$. This approach is generic, and easily allows for the incorporation of domain specific knowledge. 



\subsection{Hyperedge Representations via Hypergraph Wavelets}

To make use of the hyperedge features, we consider the dual of the hypergraph, denoted $\mathcal{G}^\star$, which is attained by interchanging the nodes and edges of the hypergraph, so $\mathcal{V}^\star = \{e_1, \ldots, e_m\}$ and $\mathcal{E}^\star = \{v_1, \ldots, v_n\}$ with incidence matrix $\mathbf{H}^\star = \mathbf{H}^T$. 
To get representations for hyperedges belonging to the original graph, the wavelets are computed based on $\mathcal{G}^\star$, and the wavelet features $\mathcal{W}_J \mathbf{z} $
are computed as described in section \ref{subsec:hypergraph_diffusion_wave}. $\mathcal{W}_J \mathbf{z} $ is used as the representation of the cellular niches. 




\section{Results}

We apply our method to the Seattle Alzheimer’s Disease Brain Cell Atlas~(SEA-AD) dataset~\cite{gabitto2023integrated} which contains MERFISH spatial transcriptomic profiling of tissues from the middle temporal gyrus of donors at different stages (as defined by neuropathologic Braak stage) of Alzheimer's disease. We find that hyperedge wavelets produce a rich representation of cellular niches that can capture disease progression, thereby enabling the identification of disease-relevant cellular niches.\footnote{Our code is available at \github.}

\begin{table}[!h]
    \centering
    \caption{Comparison among unsupervised methods. Diversity is measured by the Vendi score~\cite{dan2023vendi}, and representation quality is probed by logistic regression to predict the Braak stage.}
    \scalebox{0.74}{
    \begin{tabular}{lccccc}
    \toprule
    & Diversity & & \multicolumn{3}{c}{Linear Probing} \\
    \cmidrule{2-2} \cmidrule{4-6}
    & Vendi$\uparrow$ & & Accuracy$\uparrow$ & F1$\uparrow$ & AUROC$\uparrow$\\
    \midrule
    Raw node features &
    $7.32 \textcolor{gray}{\pm 3.23}$ && $0.74\textcolor{gray}{\pm 0.004}$ & $0.74\textcolor{gray}{\pm 0.004}$ & $0.90\textcolor{gray}{\pm 0.003}$\\
    Graph wavelets~\cite{coifman2006diffusion} & $\textbf{81.13}\textcolor{gray}{\pm 16.21}$ && $0.96\textcolor{gray}{\pm 0.002}$ & $0.96\textcolor{gray}{\pm 0.002}$ & $\textbf{1.00}\textcolor{gray}{\pm 0.000}$ \\
    Hypergraph diffusion &
    $8.26 \textcolor{gray}{\pm 2.69}$ && $0.90\textcolor{gray}{\pm 0.002}$ & $0.90\textcolor{gray}{\pm 0.002}$ & $0.98\textcolor{gray}{\pm 0.001}$ \\
    Graph AutoEncoder &
    $4.10\textcolor{gray}{\pm 0.53}$ && $0.94\textcolor{gray}{\pm 0.003}$ & $0.94\textcolor{gray}{\pm 0.003}$ & $0.99\textcolor{gray}{\pm 0.001}$  \\
    \midrule
    Hypergraph wavelets \textbf{(ours)} &
    $22.24 \textcolor{gray}{\pm 4.81}$ && $\textbf{0.97}\textcolor{gray}{\pm 0.002}$ & $\textbf{0.97}\textcolor{gray}{\pm 0.001}$ & $\textbf{1.00}\textcolor{gray}{\pm 0.001}$ \\
    \bottomrule
    \end{tabular}
    }
\label{tab:unsupervised_results}
\end{table}

\subsection{Representational Diversity and Organization}

It is desirable that representations of cellular niches reflect the complex nature of the underlying biology and be representative of disease stage. To assess these characteristics, we compute the Vendi score \cite{dan2023vendi}, a diversity metric, and fit a logistic classification model to predict the Braak stage. We contextualize the performance of hypergraph wavelet representations by comparing it with the transcriptional profiles of individual cells, graph diffusion wavelets, a graph neural network based autoencoder, and hypergraph diffusion. 

As demonstrated in Table \ref{tab:unsupervised_results}, hypergraph wavelets produce representations with a higher diversity than the graph autoencoder, one step of hypergraph diffusion, and the direct node features (transcripts at each node). Graph wavelets produce a representation with the highest Vendi score. The linear probing demonstrates that the hypergraph wavelet representations produce a latent space that is well organized by disease progression, as it achieves the highest accuracy on the classification task with an accuracy of $0.97$. The improvements of graph wavelets and hypergraph diffusion over the raw node features demonstrate the utility of wavelets and higher-order structures respectively. The two experiments lead us to hypothesize that hypergraph wavelets strike a useful balance between representational diversity and organization.  


\begin{figure}[!htb]
    \centering
    \includegraphics[width=0.4\textwidth]{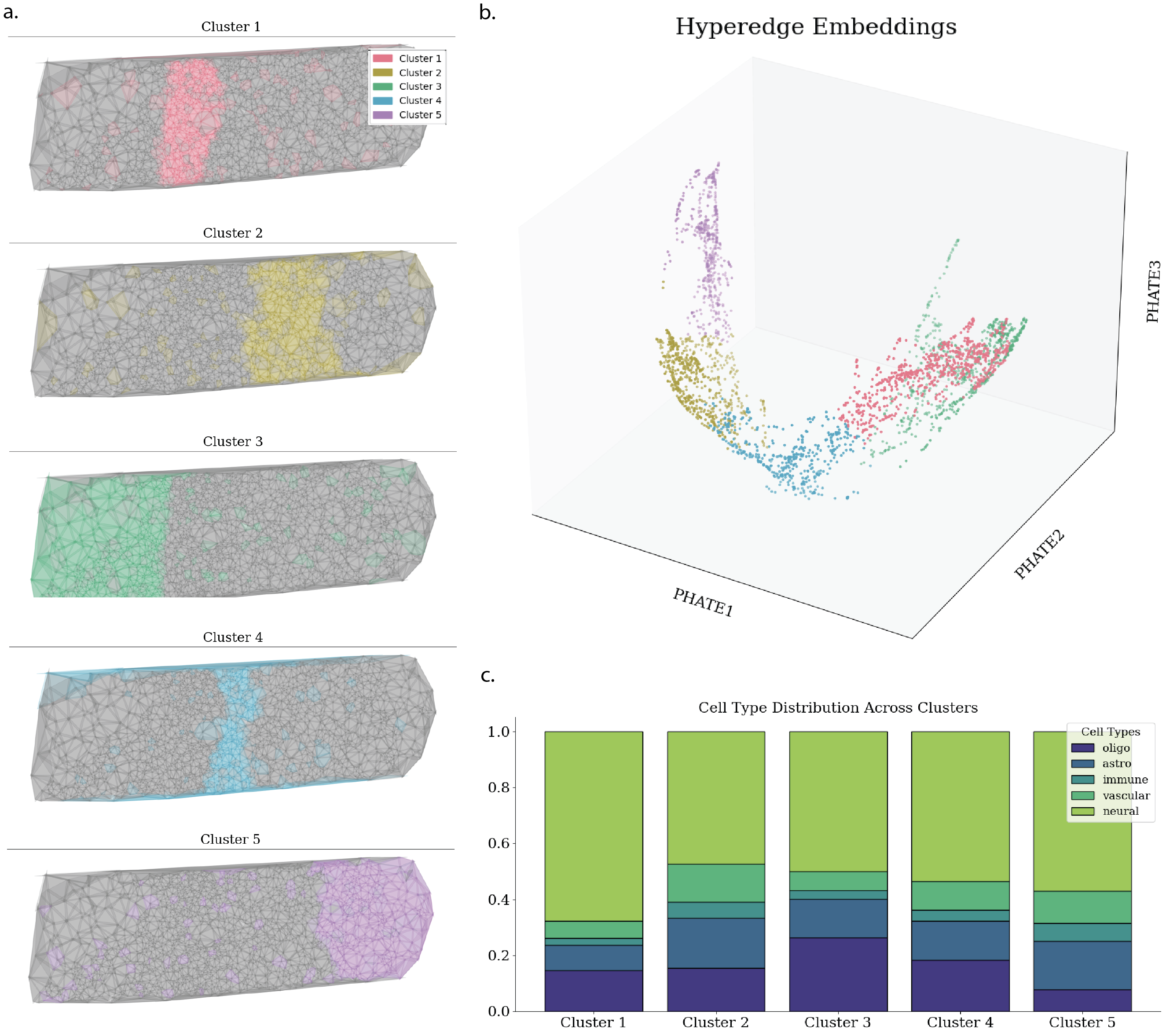}
    \caption{A. Hypergraphs clustered over tissue sample. B. PHATE \cite{moon2019visualizing} space colored by spectral clustering. C. Distribution of cell types on each cluster}
    \label{fig:discovered_cellular_niche}
\end{figure}

\subsection{Visualizing Disease Progression}

In Figure \ref{fig:discovered_cellular_niche}, it is demonstrated that cellular niches can be clustered into niche types with distinct cell compositions. While these clusters are often spatially organized, it is of interest that niches belonging to a cluster can exist across the tissue. Cellular niche representations from pairs of Braak stages are visualized together in Figure \ref{fig:braak_pairs}. The cellular niche representations have overlapping and non-overlapping regions, demonstrating the existence of shared and unshared cellular niches between Braak stages, thereby helping to identify niches that typify each individual Braak stage. 


\begin{figure}[!htb]
    \centering
    \includegraphics[width=0.5\textwidth]{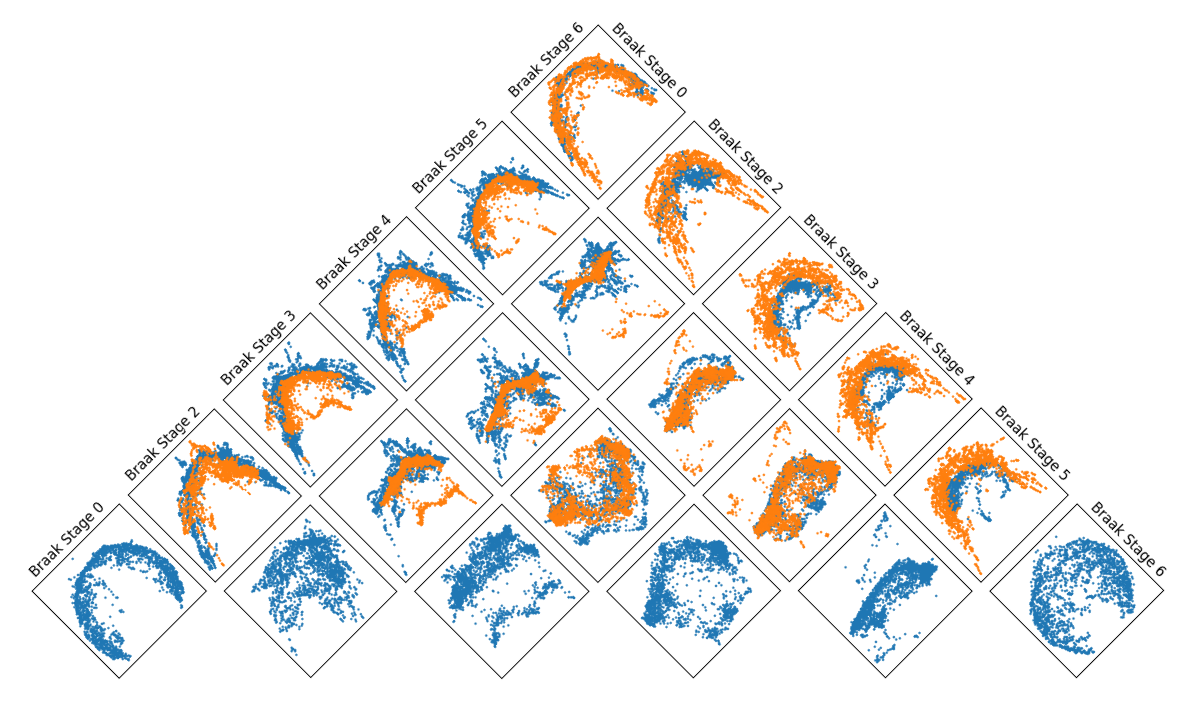}
    \caption{Visualization of pairs of cellular neighborhood representations derived from hypergraph wavelets at different Braak stages, projected into two dimensions using PHATE \cite{moon2019visualizing}. Blue points represent cellular neighborhoods at the Braak stage indicated by the column name, while orange points correspond to the stage indicated by the row name.}
    \label{fig:braak_pairs}
\end{figure}

\section{Conclusion}

\par Spatially resolved transcriptomics data captures cellular niches which consist of multiple cells in simultaneous interaction. To capture these higher order interactions, we model the spatial transcriptomics data using a hypergraph.

We propose hypergraph diffusion wavelets to perform efficient and multiscale signal processing on hypergraphs and demonstrate its utility by generating biologically meaningful representations of cellular niches in Alzheimer's patients.

\newpage
\bibliography{references}

\begin{thebibliography}{10}
\providecommand{\url}[1]{#1}
\csname url@samestyle\endcsname
\providecommand{\newblock}{\relax}
\providecommand{\bibinfo}[2]{#2}
\providecommand{\BIBentrySTDinterwordspacing}{\spaceskip=0pt\relax}
\providecommand{\BIBentryALTinterwordstretchfactor}{4}
\providecommand{\BIBentryALTinterwordspacing}{\spaceskip=\fontdimen2\font plus
\BIBentryALTinterwordstretchfactor\fontdimen3\font minus \fontdimen4\font\relax}
\providecommand{\BIBforeignlanguage}[2]{{%
\expandafter\ifx\csname l@#1\endcsname\relax
\typeout{** WARNING: IEEEtran.bst: No hyphenation pattern has been}%
\typeout{** loaded for the language `#1'. Using the pattern for}%
\typeout{** the default language instead.}%
\else
\language=\csname l@#1\endcsname
\fi
#2}}
\providecommand{\BIBdecl}{\relax}
\BIBdecl

\bibitem{node2vec}
A.~Grover and J.~Leskovec, ``node2vec: Scalable feature learning for networks,'' in \emph{Proceedings of the 22nd ACM SIGKDD international conference on Knowledge discovery and data mining}, 2016, pp. 855--864.

\bibitem{Weisfeiler_Lehman}
N.~Shervashidze, P.~Schweitzer, E.~J. Van~Leeuwen, K.~Mehlhorn, and K.~M. Borgwardt, ``Weisfeiler-lehman graph kernels.'' \emph{Journal of Machine Learning Research}, vol.~12, no.~9, 2011.

\bibitem{GCN}
T.~N. Kipf and M.~Welling, ``Semi-supervised classification with graph convolutional networks,'' \emph{International Conference on Learning Representations}, 2017.

\bibitem{GAT}
P.~Veli{\v{c}}kovi{\'c}, G.~Cucurull, A.~Casanova, A.~Romero, P.~Lio, and Y.~Bengio, ``Graph attention networks,'' \emph{International Conference on Learning Representations}, 2018.

\bibitem{GraphSAGE}
W.~Hamilton, Z.~Ying, and J.~Leskovec, ``Inductive representation learning on large graphs,'' \emph{Advances in neural information processing systems}, vol.~30, 2017.

\bibitem{HypergraphTheory}
A.~Bretto, ``Hypergraph theory,'' \emph{An introduction. Mathematical Engineering. Cham: Springer}, vol.~1, 2013.

\bibitem{HGNN}
Y.~Feng, H.~You, Z.~Zhang, R.~Ji, and Y.~Gao, ``Hypergraph neural networks,'' in \emph{Proceedings of the AAAI conference on artificial intelligence}, vol.~33, no.~01, 2019, pp. 3558--3565.

\bibitem{HGCN_HGAT}
S.~Bai, F.~Zhang, and P.~H. Torr, ``Hypergraph convolution and hypergraph attention,'' \emph{Pattern Recognition}, vol. 110, p. 107637, 2021.

\bibitem{ST_1}
P.~L. St{\aa}hl, F.~Salm{\'e}n, S.~Vickovic, A.~Lundmark, J.~F. Navarro, J.~Magnusson, S.~Giacomello, M.~Asp, J.~O. Westholm, M.~Huss \emph{et~al.}, ``Visualization and analysis of gene expression in tissue sections by spatial transcriptomics,'' \emph{Science}, vol. 353, no. 6294, pp. 78--82, 2016.

\bibitem{ST_2}
C.~G. Williams, H.~J. Lee, T.~Asatsuma, R.~Vento-Tormo, and A.~Haque, ``An introduction to spatial transcriptomics for biomedical research,'' \emph{Genome Medicine}, vol.~14, no.~1, p.~68, 2022.

\bibitem{ST_3}
L.~Moses and L.~Pachter, ``Museum of spatial transcriptomics,'' \emph{Nature methods}, vol.~19, no.~5, pp. 534--546, 2022.

\bibitem{chung1997spectral}
F.~R. Chung, \emph{Spectral graph theory}.\hskip 1em plus 0.5em minus 0.4em\relax American Mathematical Soc., 1997, vol.~92.

\bibitem{zhou2006learning}
D.~Zhou, J.~Huang, and B.~Sch{\"o}lkopf, ``Learning with hypergraphs: Clustering, classification, and embedding,'' \emph{Advances in neural information processing systems}, vol.~19, 2006.

\bibitem{tong2024learnable}
A.~Tong, F.~Wenkel, D.~Bhaskar, K.~Macdonald, J.~Grady, M.~Perlmutter, S.~Krishnaswamy, and G.~Wolf, ``Learnable filters for geometric scattering modules,'' \emph{IEEE Transactions on Signal Processing}, 2024.

\bibitem{zhang2019introducing}
S.~Zhang, Z.~Ding, and S.~Cui, ``Introducing hypergraph signal processing: Theoretical foundation and practical applications,'' \emph{IEEE Internet of Things Journal}, vol.~7, no.~1, pp. 639--660, 2019.

\bibitem{ST_tissueArch}
A.~Rao, D.~Barkley, G.~S. Fran{\c{c}}a, and I.~Yanai, ``Exploring tissue architecture using spatial transcriptomics,'' \emph{Nature}, vol. 596, no. 7871, pp. 211--220, 2021.

\bibitem{ST_tumor_morphology}
B.~He, L.~Bergenstr{\aa}hle, L.~Stenbeck, A.~Abid, A.~Andersson, {\AA}.~Borg, J.~Maaskola, J.~Lundeberg, and J.~Zou, ``Integrating spatial gene expression and breast tumour morphology via deep learning,'' \emph{Nature biomedical engineering}, vol.~4, no.~8, pp. 827--834, 2020.

\bibitem{ST_disease_traj}
D.~Pham, X.~Tan, B.~Balderson, J.~Xu, L.~F. Grice, S.~Yoon, E.~F. Willis, M.~Tran, P.~Y. Lam, A.~Raghubar \emph{et~al.}, ``Robust mapping of spatiotemporal trajectories and cell--cell interactions in healthy and diseased tissues,'' \emph{Nature communications}, vol.~14, no.~1, p. 7739, 2023.

\bibitem{ST_challenges}
S.~Fang, B.~Chen, Y.~Zhang, H.~Sun, L.~Liu, S.~Liu, Y.~Li, and X.~Xu, ``Computational approaches and challenges in spatial transcriptomics,'' \emph{Genomics, Proteomics \& Bioinformatics}, vol.~21, no.~1, pp. 24--47, 2023.

\bibitem{GraphST}
Y.~Long, K.~S. Ang, M.~Li, K.~L.~K. Chong, R.~Sethi, C.~Zhong, H.~Xu, Z.~Ong, K.~Sachaphibulkij, A.~Chen \emph{et~al.}, ``Spatially informed clustering, integration, and deconvolution of spatial transcriptomics with graphst,'' \emph{Nature Communications}, vol.~14, no.~1, p. 1155, 2023.

\bibitem{chang2022define}
Y.~Chang, F.~He, J.~Wang, S.~Chen, J.~Li, J.~Liu, Y.~Yu, L.~Su, A.~Ma, C.~Allen \emph{et~al.}, ``Define and visualize pathological architectures of human tissues from spatially resolved transcriptomics using deep learning,'' \emph{Computational and structural biotechnology journal}, vol.~20, pp. 4600--4617, 2022.

\bibitem{zahedi2024deep}
R.~Zahedi, R.~Ghamsari, A.~Argha, C.~Macphillamy, A.~Beheshti, R.~Alizadehsani, N.~H. Lovell, M.~Lotfollahi, and H.~Alinejad-Rokny, ``Deep learning in spatially resolved transcriptfomics: a comprehensive technical view,'' \emph{Briefings in Bioinformatics}, vol.~25, no.~2, p. bbae082, 2024.

\bibitem{gao2024review}
Y.~Gao, Y.-L. Gao, J.~Jing, F.~Li, C.-H. Zheng, and J.-X. Liu, ``A review of recent advances in spatially resolved transcriptomics data analysis,'' \emph{Neurocomputing}, p. 128283, 2024.

\bibitem{wu2022space}
Z.~Wu, A.~E. Trevino, E.~Wu, K.~Swanson, H.~J. Kim, H.~B. D’Angio, R.~Preska, G.~W. Charville, P.~D. Dalerba, A.~M. Egloff \emph{et~al.}, ``Space-gm: geometric deep learning of disease-associated microenvironments from multiplex spatial protein profiles,'' \emph{bioRxiv}, pp. 2022--05, 2022.

\bibitem{de2013functional}
B.~de~Bono, P.~Grenon, R.~Baldock, and P.~Hunter, ``Functional tissue units and their primary tissue motifs in multi-scale physiology,'' \emph{Journal of biomedical semantics}, vol.~4, pp. 1--13, 2013.

\bibitem{Seurat}
A.~Butler, P.~Hoffman, P.~Smibert, E.~Papalexi, and R.~Satija, ``Integrating single-cell transcriptomic data across different conditions, technologies, and species,'' \emph{Nature biotechnology}, vol.~36, no.~5, pp. 411--420, 2018.

\bibitem{SingleR}
D.~Aran, A.~P. Looney, L.~Liu, E.~Wu, V.~Fong, A.~Hsu, S.~Chak, R.~P. Naikawadi, P.~J. Wolters, A.~R. Abate \emph{et~al.}, ``Reference-based analysis of lung single-cell sequencing reveals a transitional profibrotic macrophage,'' \emph{Nature immunology}, vol.~20, no.~2, pp. 163--172, 2019.

\bibitem{scBERT}
F.~Yang, W.~Wang, F.~Wang, Y.~Fang, D.~Tang, J.~Huang, H.~Lu, and J.~Yao, ``scbert as a large-scale pretrained deep language model for cell type annotation of single-cell rna-seq data,'' \emph{Nature Machine Intelligence}, vol.~4, no.~10, pp. 852--866, 2022.

\bibitem{wang2022benchmarking}
Y.~Wang, X.~Sun, and H.~Zhao, ``Benchmarking automated cell type annotation tools for single-cell atac-seq data,'' \emph{Frontiers in Genetics}, vol.~13, p. 1063233, 2022.

\bibitem{palla2022squidpy}
G.~Palla, H.~Spitzer, M.~Klein, D.~Fischer, A.~C. Schaar, L.~B. Kuemmerle, S.~Rybakov, I.~L. Ibarra, O.~Holmberg, I.~Virshup \emph{et~al.}, ``Squidpy: a scalable framework for spatial omics analysis,'' \emph{Nature methods}, vol.~19, no.~2, pp. 171--178, 2022.

\bibitem{gabitto2023integrated}
M.~I. Gabitto, K.~J. Travaglini, V.~M. Rachleff, E.~S. Kaplan, B.~Long, J.~Ariza, Y.~Ding, J.~T. Mahoney, N.~Dee, J.~Goldy \emph{et~al.}, ``Integrated multimodal cell atlas of alzheimer’s disease,'' \emph{Research Square}, 2023.

\bibitem{dan2023vendi}
D.~Dan~Friedman and A.~B. Dieng, ``The vendi score: A diversity evaluation metric for machine learning,'' \emph{Transactions on machine learning research}, 2023.

\bibitem{coifman2006diffusion}
R.~R. Coifman and S.~Lafon, ``Diffusion maps,'' \emph{Applied and computational harmonic analysis}, vol.~21, no.~1, pp. 5--30, 2006.

\bibitem{moon2019visualizing}
K.~R. Moon, D.~Van~Dijk, Z.~Wang, S.~Gigante, D.~B. Burkhardt, W.~S. Chen, K.~Yim, A.~v.~d. Elzen, M.~J. Hirn, R.~R. Coifman \emph{et~al.}, ``Visualizing structure and transitions in high-dimensional biological data,'' \emph{Nature biotechnology}, vol.~37, no.~12, pp. 1482--1492, 2019.

\end{thebibliography}
\bibliographystyle{IEEEtran}

\end{document}